\documentclass{article}

\usepackage{amsmath}
\usepackage{amssymb}
\usepackage{amsthm}
\usepackage{tikz}
\usepackage{url}

\newcounter{global}
\theoremstyle{definition}
\newtheorem{definition}[global]{Definition}
\theoremstyle{plain}
\newtheorem{theorem}[global]{Theorem}
\newtheorem{lemma}[global]{Lemma}
\newtheorem{corollary}[global]{Corollary}
\newtheoremstyle{note}{}{}{}{}{\itshape}{.}{.5em}{}
\theoremstyle{note}
\newtheorem{remark}{Remark}%
\newtheorem{example}{Example}%
\renewcommand\section{%
  \@startsection {section}{1}{\z@}%
  {-3.5ex \@plus -1ex \@minus -.2ex}%
  {2.3ex \@plus.2ex}%
  {\normalfont\large\bfseries}}
\pdfmapline{=eurmo10 EURM10 ".167 SlantFont" <eurm10.pfb}
\pdfmapline{=eurmo10 EURM10 " .167 SlantFont" <eurm10.pfb}

\DeclareFontFamily{U}{eurmo}{}
\DeclareFontShape{U}{eurmo}{m}{n}{<-6>eurmo10<6-8>eurmo10<8->eurmo10}{}
\DeclareMathAlphabet{\eurmo}{U}{eurmo}{m}{n}

\begin{document}

\title{On minimal sets of graded attribute implications}

\date{\normalsize%
  Dept. Computer Science, Palacky University, Olomouc}

\author{Vilem Vychodil\footnote{%
    e-mail: \texttt{vychodil@binghamton.edu},
    phone: +420 585 634 705,
    fax: +420 585 411 643}}

\maketitle

\begin{abstract}
  We explore the structure of non-redundant and minimal sets consisting
  of graded if-then rules. The rules serve as graded attribute implications
  in object-attribute incidence data and as similarity-based functional
  dependencies in a similarity-based generalization of the relational model
  of data. Based on our observations, we derive
  a polynomial-time algorithm which transforms a given finite set of rules
  into an equivalent one which has the least size in terms of the number of
  rules.
\end{abstract}

\newcommand{\I}{\ensuremath{\Rightarrow}}
\newcommand{\PrEq}[3][T]{\ensuremath{#2 \equiv_{#1} #3}}
\newcommand{\symbPrEq}[1][T]{\equiv_{#1}}
\newcommand{\Cl}[2][T]{\ensuremath{[#2]_{#1}}}
\newcommand{\ET}[2][T]{\ensuremath{\mathrm{E}_{#1}(#2)}}
\newcommand{\ETs}[1][T]{\ensuremath{\mathcal{E}_{#1}}}
\newcommand{\et}[2][T]{\ensuremath{\mathrm{e}_{#1}(#2)}}
\newcommand{\dipr}[1][T]{#1 \Vdash}
\newcommand{\ndipr}[1][T]{#1 \nVdash}

%%%%%%%%%%%%%%%%%%%%%%%%%%%%%%%%%%%%%%%%%%%%%%%%%%%%%%%%%%%%%%%%%%%%%%%%%%%%%%%%
%%%%%   INTRODUCTION
%%%%%%%%%%%%%%%%%%%%%%%%%%%%%%%%%%%%%%%%%%%%%%%%%%%%%%%%%%%%%%%%%%%%%%%%%%%%%%%%
\section{Introduction}\label{sec:intro}
Reasoning with various types of if-then rules is crucial in many disciplines
ranging from theoretical computer science to applications. Among the most
widely used rules are those taking from of implications between
conjunctions of attributes. Such rules are utilized in database systems
(as functional dependencies or inclusion dependencies~\cite{Mai:TRD}),
logic programming
(as particular definite clauses representing programs~\cite{Lloyd84}),
and data mining
(as attribute implications~\cite{GaWi:FCA} or
association rules~\cite{AgImSw:ASR,Zak:Mnrar}).
One of the most important
problems regarding the rules is to find for a given set $T$ of rules a set
of rules which is equivalent to $T$ and minimal in terms of its size.
In relational database theory~\cite{Mai:TRD}, the problem is referred to
as finding minimal covers of $T$.

In this paper, we deal with the problem of finding minimal and equivalent
sets of rules for general rules describing dependencies between
\emph{graded attributes.}
That is, instead of the classic rules which are
often considered as implications
\begin{align}
  \{y_1,\ldots,y_m\} \I \{z_1,\ldots,z_n\}
  \label{eqn:ord}
\end{align}
between sets of attributes, describing presence/absence of attributes,
we deal with rules where the presence/absence of attributes is
expressed to degrees.
That is, the rules in question can be written as
\begin{align}
  \bigl\{{}^{a_1\!}/y_1,\ldots,{}^{a_m\!}/y_m\bigr\} \I
  \bigl\{{}^{b_1\!}/z_1,\ldots,{}^{b_n\!}/z_n\bigr\}
  \label{eqn:grad}
\end{align}
and understood as rules saying that ``if $y_1$ is present at least
to degree $a_1$ and $\cdots$ and $y_m$ is present at least
to degree $a_m$, then $z_1$ is present at least
to degree $b_1$ and $\cdots$ and $z_n$ is present at least
to degree $b_n$.'' We assume that the degrees appearing
in~\eqref{eqn:grad} come from a structure of truth degrees which is
more general than the two-element Boolean algebra and allows for
\emph{intermediate degrees of truth.}
In particular, we use complete residuated lattices~\cite{GaJiKoOn:RL}
with linguistic hedges~\cite{EsGoNo:Hedges,Haj:Ovt,Za:Afstilh}
for the job. In our setting, \eqref{eqn:grad} can be seen as generalization
of~\eqref{eqn:ord} if all the degrees $a_1,\ldots,b_1,\ldots$ are equal
to $1$ (as usual, $1$ denotes the classical truth value of ``full truth'').

Our previous results on rules of the form~\eqref{eqn:grad} include
a fixed point characterization of a semantic entailment,
Armstrong-style~\cite{Arm:Dsdbr} axiomatizations in the ordinary style
and the graded style (also known as Pavelka-style completeness,
see~\cite{Pav:Ofl1,Pav:Ofl2,Pav:Ofl3}), results on generating non-redundant
bases from data, and two kinds of semantics of the rules:
(i) a \emph{database semantics} which is based on evaluating the rules
in ranked data tables over domains with similarities~\cite{BeVy:DASFAA},
and (ii) an \emph{incidence data semantics} which is based on evaluating
the rules in object-attribute data tables with graded
attributes~\cite{Bel:FRS,BeVy:Fcalh} which are known as
\emph{formal contexts} in formal concept analysis~\cite{GaWi:FCA}.
Analogously as for the ordinary rules, one can show that both (i) and (ii)
yield the same notion of semantic entailment which simplifies further
considerations, e.\,g., a single axiomatization of the semantic entailment
works for both the database and incidence data semantics of the rules.
A survey of recent results regarding the rules can be
found in~\cite{BeVy:ADfDwG}.

In this paper, we consider rules like~\eqref{eqn:grad} and explore the
structure of non-redundant and minimal sets of rules of this type.
We show an if-and-only-if criterion of minimality and a polynomial-time
procedure which, given $T$, transforms $T$ into an equivalent and minimal
set of graded rules. Let us note that the previous results regarding
minimality of sets of graded rules~\cite{BeVy:ADfDwG} were focused exclusively
on sets of rules generated from data. That is, the input for such
instance-based approaches is not a set $T$ of rules. Instead, the input
is assumed to be a structure (e.g., a formal context with graded attributes
or a database table over domains with similarities) and the goal is to find
a minimal set $T$ of rules which entails exactly all
the rules true in the structure.
One particular example is an algorithm for generating graded
counterparts to Guigues-Duquenne bases~\cite{GuDu}
described in~\cite{BeVy:ADfDwG}.
In contrast, the problem studied in this paper is different. We assume
that a set $T$ of rules is already given (e.g., inferred from data or
proposed by an expert) but it may not be minimal. Therefore, it is interesting
to find a minimal set of rules which conveys the same information.
Unlike the instance-based methods which belong to hard problems~\cite{DiSe:OCEP}
even for the classic (non-graded) rules, the minimization method
presented in this paper is polynomial and therefore tractable.

The present paper is organized as follows. Section~\ref{sec:prelim}
presents preliminaries from structures of degrees and graded if-then rules.
Section~\ref{sec:results} contains the new results.

%%%%%%%%%%%%%%%%%%%%%%%%%%%%%%%%%%%%%%%%%%%%%%%%%%%%%%%%%%%%%%%%%%%%%%%%%%%%%%%%
%%%%%   INTRODUCTION
%%%%%%%%%%%%%%%%%%%%%%%%%%%%%%%%%%%%%%%%%%%%%%%%%%%%%%%%%%%%%%%%%%%%%%%%%%%%%%%%
\section{Preliminaries}\label{sec:prelim}
In this section, we present basic notions from structures of truth degrees
and graded attribute implications which formalize rules like~\eqref{eqn:grad}.
We only present the notions and results which are sufficient to follow the
results in Section~\ref{sec:results}. Interested readers may find more results
in~\cite{Bel:FRS,BeVy:ADfDwG,GaJiKoOn:RL,Gog:Lic,Haj:MFL,Ho:ML}.

A (complete) residuated lattice \cite{Bel:FRS,GaJiKoOn:RL} is an algebra
$\mathbf{L}=\langle L,\wedge,\vee,\otimes,\rightarrow,0,1\rangle$ where
$\langle L,\wedge,\vee,0,1 \rangle$ is a (complete) lattice,
$\langle L,\otimes,1 \rangle$ is a commutative monoid, and
$\otimes$ (multiplication, a truth function of ``fuzzy conjunction'')
and $\rightarrow$ (residuum, a truth function of ``fuzzy implication'')
satisfy the adjointness property:
$a \otimes b \leq c$ if{}f $a \leq b \rightarrow c$
($a,b,c \in L$). Examples of complete residuated
lattices include structures on the real unit interval given
by left-continuous t-norms~\cite{EsGo:MTL,Haj:MFL} as well as
finite structures of degrees.

If $U \ne \emptyset$, we can consider the direct
power $\mathbf{L}^U =
\langle L^U,\cap,\cup,\otimes,\rightarrow,{}^{\ast},\emptyset_U,1_U\rangle$
of $\mathbf{L}$. Each $A \in L^U$ is called an $\mathbf{L}$-set
(s fuzzy set) $A$ in universe $U$. That is, $A \in L^U$ is a map
$A\!:U \to L$, $A(u)$ being interpreted as
``the degree to which $u$ belongs to $A$''.
Operations $\cap,\cup,\otimes,\dots$ in $\mathbf{L}^U$
represent operations with
$\mathbf{L}$-sets which are induced by the corresponding operations
$\wedge,\vee,\otimes,\dots$ in $\mathbf{L}$. Hence, e.g., 
$(A \cup B)(u) = A(u) \vee B(u)$ for each $u \in U$.
Note that for the lattice order $\subseteq$ in $\mathbf{L}^U$ being
induced by $\leq$, we have $A \subseteq B$ if{}f, for each $u \in U$,
$A(u) \leq B(u)$. Therefore, $A \subseteq B$ denotes ``full containment''
of $A$ in $B$.
If $U = \{u_1,\dots,u_n\}$ ($U$ is finite),
we adopt the usual conventions for writing
$\mathbf{L}$-sets $A \in L^U$ as
$\{{}^{a_1\!}/u_1,\dots,{}^{a_n\!}/u_n\}$ meaning
that $A(u_i) = a_i$ ($i=1,\dots,n$). Furthermore, in the notation
we omit ${}^{a_i\!}/u_i$ if $a_i = 0$
and write $u_i$ if $a_i = 1$.

Let $Y$ be a finite non-empty set of attributes (i.e., symbolic names).
A~\emph{graded attribute implication} in $Y$ is an expression $A \I B$,
where $A,B \in L^Y$. In our paper, graded attribute implications are
regarded as formulas representing rules like~\eqref{eqn:grad}.
The interpretation of graded attribute implications
is based on the notion of a \emph{graded subsethood of $\mathbf{L}$-sets}
in a similar way as the interpretation of the ordinary attribute
implications~\cite{GaWi:FCA} is based on the ordinary subsethood.
In a more detail, for any $A,M \in L^Y$, we define a degree $S(A,M) \in L$
of subsethood of $A$ in $M$ by
\begin{align}
  S(A,M) = \textstyle{\bigwedge}_{y \in Y}\bigl(A(y) \rightarrow M(y)\bigr).
  \label{Eq:S}
\end{align}
Clearly, $A \subseteq M$ (i.e., $A$ is fully contained in $M$)
if{}f $S(A,M) = 1$. For any $A,B,M \in L^Y$, we may put
\begin{align}
  ||A \I B||_M &=
  \left\{
    \begin{array}{@{\,}l@{\quad}l@{}}
      S(B,M), & \text{if } A \subseteq M, \\
      1, & \text{otherwise,}
    \end{array}
  \right.
  \label{eqn:glob_sem}
\end{align}
and call $||A \I B||_M$ a degree to which $A \I B$ is true in $M$.
Therefore, if $M$ is regarded as an $\mathbf{L}$-set of attributes of
an object with each $M(y)$ interpreted as the degree to which the object
has attribute $y$, then $||A \I B||_M$ is a degree to which the following
statement is true: ``If the object has all the attributes from $A$, then
it has all the attributes from $B$''. Interestingly, \eqref{eqn:glob_sem}
is not the only possible (and reasonable) interpretation of $A \I B$ in $M$.
In fact, our approach in~\cite{BeVy:ADfDwG} is more general in that it defines
$||A \I B||^*_M$ by
\begin{align}
  ||A \I B||^*_M &= S(A,M)^* \rightarrow S(B,M),
  \label{eqn:impl_general}
\end{align}
where ${}^*$ is an idempotent truth-stressing linguistic
hedge~\cite{Za:Afstilh,Za:lv1,Za:lv2,Za:lv3}
on~$\mathbf{L}$ (shortly, a hedge).
We assume that ${}^*$ is a map ${}^*\!: L \to L$ such that
(i) $1^{\ast} = 1$,
(ii) $a^{\ast} \leq a$,
(iii) $(a \rightarrow b)^{\ast} \leq a^{\ast} \rightarrow b^{\ast}$, and
(iv) $a^{\ast\ast} = a^{\ast}$ ($a,b \in L$).
A hedge ${}^{\ast}$ can
be seen as a generalization of Baaz's $\Delta$ operation~\cite{Baaz,Haj:MFL}
and it has been introduced in
fuzzy logic in the narrow sense~\cite{Got:Mfl}
by H\'ajek in~\cite{Haj:Ovt}.
In the sense of~\cite{Haj:Ovt}, ${}^{\ast}$ can be seen
as a truth function for unary logical connective ``very true''.

Now, one can see that~\eqref{eqn:glob_sem} is a particular case
of~\eqref{eqn:impl_general} for ${}^*$ being the so-called
globalization~\cite{TaTi:Gist}:
\begin{equation}
  a^{\ast} = \left\{
    \begin{array}{@{\,}l@{\quad}l}
      1, & \mbox{if}\ a = 1, \\
      0, & \mbox{otherwise.}
    \end{array}
  \right.
  \label{eqn:glob}
\end{equation}
Indeed, for ${}^*$ introduced by~\eqref{eqn:glob},
we have either $a^* \rightarrow b = 1$ if $a < 1$
or $a^* \rightarrow b = b$ if $a = 1$ and thus 
\eqref{eqn:impl_general} becomes \eqref{eqn:glob_sem}.
Note that in case of linearly ordered structures of truth degrees,
globalization coincides with Baaz's $\Delta$ operation (this is not true
in general). On the other hand, if ${}^*$ is identity, then
the right-hand side of \eqref{eqn:impl_general} becomes 
\begin{align}
  S(A,M) \rightarrow S(B,M),
  \label{eqn:impl_id}
\end{align}
which may also be regarded as a desirable interpretation of $A \I B$ in $M$.
The approach via hedges in~\cite{BeVy:ADfDwG} allows us to deal with both
\eqref{eqn:glob_sem} and \eqref{eqn:impl_id} (and possibly other
interpretations resulting by the choice of other hedges) in a unified way
because \eqref{eqn:glob_sem} and \eqref{eqn:impl_id} result as two borderline
choices of hedges in~\eqref{eqn:impl_general}, namely, the globalization and
the identity on $L$. Also, since ${}^*$ may be interpreted as a truth function
of logical connective ``very true'', the general degree
$||A \Rightarrow B||^*_M$ introduced by~\eqref{eqn:impl_general} may be
interpreted as a truth degree of the following statement:
``If it is \emph{very true} that the object
(whose attributes are represented by $M$) has all the attributes from $A$,
then it has all the attributes from $B$''.
Therefore, we may view the hedge as a parameter of the interpretation
of graded attribute implications, see~\cite{BeVy:ADfDwG} for a detailed
explanation and further remarks on the role of hedges. Recent results
on hedges and their treatment in fuzzy logics in the narrow sense can
be found in~\cite{EsGoNo:Hedges}.

For graded attribute implications, we introduce a semantic and a syntactic
entailment (a provability) as follows.
A set $T$ of graded attribute implications (in $Y$) is called a theory (in $Y$).
An $\mathbf{L}$-set $M \in L^Y$ is called a model of $T$ if $||A \I B||^*_M = 1$
for all $A \I B \in T$. Let $\mathrm{Mod}(T)$ denote the set of
all models of $T$. The
\emph{degree $||A \I B||^*_T$ to which $A \I B$
  is semantically entailed by $T$} is defined by
\begin{align}
  ||A \I B||^*_T &=
  \textstyle\bigwedge_{M \in \mathrm{Mod}(T)}||A \I B||^*_M.
\end{align}
Put in words, $||A \I B||^*_T$ is a degree to which $A \I B$
is true in all models of~$T$. A graded attribute implication $A \I B$
is called \emph{trivial} whenever $||A \I B||^*_\emptyset = 1$.

The syntactic entailment of graded attribute implications is based on
an Armstrong-style axiomatic system~\cite{Arm:Dsdbr}.
Namely, each $A {\cup} B \I A$ ($A,B \in L^Y$) is considered as
an \emph{axiom} and we consider the following
\emph{deduction rules}~\cite{BeVy:ADfDwG}:
\bgroup%
\addtolength{\leftmargini}{1.2ex}%
\begin{itemize}\parskip=0pt%
\item[(Cut)]
  from $A \I B$ and $B{\cup}C \I D$ infer $A{\cup}C \I D$,
\item[(Mul)]
  from $A \I B$ infer $c^*{\otimes}A \I c^*{\otimes}B$,
\end{itemize}
\egroup%
\noindent%
where $A,B,C,D\in L^Y$, $c\in L$, and $c^*{\otimes}A$
(and analogously $c^*{\otimes}B$) denotes
the so-called $c^*$-multiple of $A \in L^Y$ which is an $\mathbf{L}$-set
such that $(c^*{\otimes}A)(y) = c^* \otimes A(y)$ for all $y \in Y$.
Note that in database literature, the classic counterpart to (Cut)
is known under the name pseudo-transitivity, see~\cite{Mai:TRD}.
The name \emph{cut} comes from~\cite{Hol}.
A proof of $A \I B$ from $T$ is a sequence $\varphi_1,\ldots,\varphi_n$
such that $\varphi_n$ equals $A \I B$ and for each $\varphi_i$ we either
have $\varphi_i \in T$, or $\varphi_i$ is an axiom, or $\varphi_i$ is
derived by (Cut) or (Mul) from some of $\varphi_1,\ldots,\varphi_{i-1}$.
A graded attribute implication \emph{$A \I B$ is provable from $T$,}
denoted $T \vdash A \I B$ if there is a proof of $A \I B$ from $T$.
In the paper, we utilize properties of $\vdash$ called the additivity,
projectivity, and transitivity, i.e., we use the facts that
\bgroup%
\addtolength{\leftmargini}{1.8ex}%
\begin{itemize}\parskip=0pt%
\item[(Add)]
  $\{A \I B, A \I C\} \vdash A \I B{\cup}C$,
\item[(Pro)]
  $\{A \I B{\cup}C\} \vdash A \I B$,
\item[(Tra)]
  $\{A \I B,B \I C\} \vdash A \I C$.
\end{itemize}
\egroup%
\noindent%
for all $A,B,C \in L^Y$, see~\cite{BeVy:ADfDwG}.

\begin{remark}
  Let us note that the trivial graded attribute implications are
  exactly the axioms, i.e., all graded attribute implications
  which are true in all models to degree $1$ are of the form 
  $A {\cup} B \I A$. Also note that if ${}^*$ is~\eqref{eqn:glob},
  then (Mul) becomes a trivial deduction rule and can be disregarded.
  Let us stress that $\cup$ in the above expressions denotes the
  operation in $\mathbf{L}_Y$ induced by $\vee$ in $\mathbf{L}$.
  Therefore, the antecedent of $A {\cup} B \I A$
  should be read ``the union of $A$ and $B$'', etc.
\end{remark}

For each $A \in L^Y$, the least model of $T$ which contains $A \in L^Y$
is called the (\emph{semantic}) \emph{closure} of $A$ and is denoted by
$\Cl{A}$. For each $A \in L^Y$ and $T$, $\Cl{A}$ always exists since
the set of all models of $T$ is closed under arbitrary intersections.
The following ordinary-style~\cite{Haj:MFL} completeness theorem
is established:

\begin{theorem}[completeness, see \cite{BeVy:ADfDwG}]\label{th:compl}%
  Let $\mathbf{L}$ and $Y$ be finite.
  Then, for any $T$ and $A,B \in L^Y$,
  the following conditions are equivalent:
  \begin{itemize}\parskip=0pt
  \item[\rm (i)]
    $T \vdash A \I B$,
  \item[\rm (ii)]
    $B \subseteq \Cl{A}$,
  \item[\rm (iii)]
    $||A \I B||^*_T = 1$.
    \qed
  \end{itemize}
\end{theorem} 

Taking into account Theorem~\ref{th:compl}, we may freely interchange
the semantic entailment (to degree $1$) and provability on condition that
both $\mathbf{L}$ and $Y$ are finite which we assume from now on---cases
of infinite $\mathbf{L}$ can be handled by adding an infinitary rule but
the issue is not relevant to this paper, cf.~\cite{BeVy:ADfDwG}.
Theory $T$ is \emph{non-redundant} if $T \setminus \{A \I B\} \nvdash A \I B$
for all $A \I B \in T$. Theories $T_1$ and $T_2$ are \emph{equivalent} if,
for all $A,B \in L^Y$,
$T_1 \vdash A \I B$ if{}f $T_2 \vdash A \I B$.

\begin{remark}
  (a)
  Alternative graph-based proof systems~\cite{UrVy:Dddosbd} as well as
  automated provers based on simplification equivalences as in~\cite{SL}
  are also available. Let us also note that in addition to
  Theorem~\ref{th:compl} which provides a syntactic characterization 
  only for formulas which are semantically entailed to degree $1$,
  the logic of graded attribute implications is also complete in
  the graded style (Pavelka-style completeness). Namely,
  $||A \I B||^*_T = \bigvee\{c \in L \,|\, T \vdash A \I c{\otimes}B\}$,
  cf.~\cite{BeVy:ADfDwG}.

  (b)
  The general interpretation of $A \Rightarrow B$ in $M$ introduced 
  in~\eqref{eqn:impl_general} corresponds to the incidence data semantics
  we have mentioned in the introduction. There are alternative
  interpretations which yield the same notion of semantic entailment.
  For instance, instead of $M$, one may take (ranked) data tables over
  domains with similarities and define the interpretation of 
  $A \Rightarrow B$ in such structures. In effect, the graded implications
  interpreted this way can be seen as similarity-based functional dependencies
  and play analogous role to the ordinary functional dependencies in the
  classical relational model of data. Since the database and incidence
  data semantics yield the same notions of semantic entailment and thus
  the same complete axiomatization, we refrain from commenting on further
  details. Interested readers may check~\cite{BeVy:ADfDwG}.
\end{remark}

%%%%%%%%%%%%%%%%%%%%%%%%%%%%%%%%%%%%%%%%%%%%%%%%%%%%%%%%%%%%%%%%%%%%%%%%%%%%%%%%
%%%%%   INTRODUCTION
%%%%%%%%%%%%%%%%%%%%%%%%%%%%%%%%%%%%%%%%%%%%%%%%%%%%%%%%%%%%%%%%%%%%%%%%%%%%%%%%
\section{Results}\label{sec:results}
Recall that procedures for removing redundancy from theories
are well known~\cite{BeVy:ADfDwG}.
That is, given a finite theory $T$, one may compute $T' \subseteq T$ which
is equivalent to $T$ and which is in addition non-redundant. Indeed,
according to Theorem~\ref{th:compl},
$T$ is redundant if{}f there is $A \Rightarrow B \in T$ such that
$B \subseteq \Cl[T \setminus \{A \Rightarrow B\}]{A}$ in which case
one can remove $A \Rightarrow B$ from $T$ and repeat the procedure
until $T$ becomes non-redundant. This procedure can be used to remove
all formulas which make $T$ redundant but it does not guarantee that
the result is minimal in terms of the number of formulas in $T$. In this paper,
we show one approach to deal with the issue.

For practical reasons, we limit ourselves to \emph{finite theories.}
Otherwise, in general we would not be able to transform a theory into
an equivalent and minimal one in finitely many steps. Furthermore,
we assume that ${}^*$ is \emph{globalization}, i.e., $||A \Rightarrow B||^*_M$
is in fact given by~\eqref{eqn:glob_sem} and (Mul) can be omitted.
In the text, we give counterexamples indicating that the present theory
cannot be directly generalized for general hedges at least not with
a substantial modification. Interestingly, the instance-based approaches
have an analogous practical limitation, cf.~\cite{BeVy:ADfDwG}.

We start by presenting a technical observation on
the properties of provability which also depends on ${}^*$
being the globalization.

\begin{theorem}\label{th:AIC}
  Let $T$ be a theory such that $T \vdash A \I B$ and
  $T \setminus \{C \I D\} \nvdash A \I B$.
  Then, $T \setminus \{C \I D\} \vdash A \I C$.
\end{theorem}
\begin{proof}
  Observe that $T \vdash A \I B$ and $T \setminus \{C \I D\} \nvdash A \I B$
  means that each proof of $A \I B$ by $T$ contains $C \I D$.
  Using properties
  of closures, we get $B \nsubseteq \Cl[T \setminus \{C \I D\}]{A}$,
  $B \subseteq \Cl{A}$, and $\Cl[T \setminus \{C \I D\}]{A} \subseteq \Cl{A}$
  which altogether yield $\Cl[T \setminus \{C \I D\}]{A} \subset \Cl{A}$.
  In order to prove $T \setminus \{C \I D\} \vdash A \I C$, it suffices to show
  that $C \subseteq \Cl[T \setminus \{C \I D\}]{A}$. By contradiction, assume
  that $C \nsubseteq \Cl[T \setminus \{C \I D\}]{A}$. By definition,
  $\Cl[T \setminus \{C \I D\}]{A}$ is the least model of
  $T \setminus \{C \I D\}$ containing $A$.
  Now, observe that since ${}^*$ is globalizaton, it follows
  that~\eqref{eqn:impl_general} becomes~\eqref{eqn:glob_sem} and thus 
  $||C \I D||^*_{\Cl[T \setminus \{C \I D\}]{A}} = 1$ because
  $C \nsubseteq \Cl[T \setminus \{C \I D\}]{A}$.
  As a consequence, $\Cl[T \setminus \{C \I D\}]{A} \in \mathrm{Mod}(T)$
  because $T$ and $T \setminus \{C \I D\}$ differ only in
  the presence of $C \I D$ in $T$. Hence, $\Cl[T \setminus \{C \I D\}]{A}$
  is a model of $T$ which contains $A$ and since $\Cl{A}$ is the least
  model of $T$ containing $A$,
  we must have $\Cl{A} \subseteq \Cl[T \setminus \{C \I D\}]{A}$,
  which contradicts the fact
  that $\Cl[T \setminus \{C \I D\}]{A} \subset \Cl{A}$.
  Therefore, we have $C \subseteq \Cl[T \setminus \{C \I D\}]{A}$ which gives
  $T \setminus \{C \I D\} \vdash A \I C$.
\end{proof}

\begin{example}\label{ex:1}
  Let us show that Theorem~\ref{th:AIC} does not hold in the case of general
  hedges. For instance, let $\mathbf{L}$ be a three-element equidistant
  subchain of the standard \L ukasiewicz algebra with ${}^*$ being the
  identity. That is, $L = \bigl\{0,0.5,1\bigr\}$,
  $\wedge$ and $\vee$ coincide with maximum and minimum,
  respectively, and
  \begin{align*}
    a \otimes b &= \max(0,a+b-1), \\
    a \rightarrow b &= \min(1,1-a+b),
  \end{align*}
  for all $a,b \in L$.
  Consider $T = \bigl\{\{y\} \I \{z\}\bigr\}$. Obviously,
  $T \vdash \{{}^{0.5\!}/y\} \I \{{}^{0.5\!}/z\}$ on
  account of (Mul). In addition, $\{{}^{0.5\!}/y\} \I \{{}^{0.5\!}/z\}$
  is non-trivial and thus not provable by the empty set of formulas.
  Furthermore, $T \nvdash \{{}^{0.5\!}/y\} \I \{y\}$. Indeed,
  for $M = \{{}^{0.5\!}/y,{}^{0.5\!}/z\}$, we get
  \begin{align*}
    S\bigl(\{y\},M\bigr) =
    (1 \rightarrow 0.5) \wedge (0 \rightarrow 0.5) =
    0.5 \wedge 1 = 0.5 = 
    S\bigl(\{z\},M\bigr),
  \end{align*}
  i.e., $M \in \mathrm{Mod}(T)$. On the other hand, 
  \begin{align*}
    S\bigl(\{{}^{0.5\!}/y\},M\bigr) =
    (0.5 \rightarrow 0.5) \wedge (0 \rightarrow 0.5) =
    1 \wedge 1 =
    1 \nleq 0.5 = 
    S\bigl(\{y\},M\bigr),
  \end{align*}
  i.e., $||\{{}^{0.5\!}/y\} \I \{y\}||^*_M < 1$. Hence, due to 
  soundness, $T \nvdash \{{}^{0.5\!}/y\} \I \{y\}$ which illustrates
  that in case of general hedges, one cannot always conclude
  $T \setminus \{C \I D\} \vdash A \I C$ provided that 
  $T \vdash A \I B$ and $T \setminus \{C \I D\} \nvdash A \I B$.
\end{example}

As a consequence of Theorem~\ref{th:AIC}, we can prove an assertion which
gives us new insight into the structure of non-redundant theories. The
assertion matches formulas from non-redundant theories based on the
following notions of equivalence:

\begin{definition}\label{def:PrEq}
  Let $T$ be a theory and $A,B \in L^Y$. We say that $A$ and $B$
  are provably equivalent under $T$, written $\PrEq{A}{B}$,
  whenever $T \vdash A \I B$ and $T \vdash B \I A$.
\end{definition}

Obviously, $\symbPrEq$ is an equivalence relation on $L^Y$. Using the notion,
we establish the following observation.

\begin{theorem}\label{th:PrEq}
  Let $T_1$ and $T_2$ be equivalent and non-redundant.
  Then, for each $A \I B \in T_1$ there is $C \I D \in T_2$
  such that $\PrEq[T_1]{A}{C}$.
\end{theorem}
\begin{proof}
  Take $A \I B \in T_1$. Since $T_1$ and $T_2$ are equivalent,
  we get $T_2 \vdash A \I B$. Therefore, there is
  a proof of $A \I B$ by $T_2$ which uses finitely many
  formulas $\varphi_1,\dots,\varphi_k$ in $T_2$.
  In addition, we can select a subset
  $T'_2 \subseteq \{\varphi_1,\dots,\varphi_k\}$ such that
  $T'_2 \vdash A \I B$ and $T'' \nvdash A \I B$
  for all $T'' \subset T'_2$. Observe that since $A \I B \in T_1$ and
  $T_1$ is non-redundant, then $A \I B$ is non-trivial. Therefore,
  $T'_2$ is non-empty. Clearly, we get $T_1 \vdash \varphi$
  for each $\varphi \in T'_2$ since $T_1$ and $T_2$ are equivalent.

  We now claim that there is $C \I D \in T'_2$ such that 
  each proof of $C \I D$ by $T_1$ contains $A \I B$, i.e.,
  $T_1 \setminus \{A \I B\} \nvdash C \I D$.
  By contradiction, assume that
  $T_1 \setminus \{A \I B\} \vdash C \I D$
  for each $C \I D \in T'_2$.
  Since $T'_2 \vdash A \I B$, we would get that
  $T_1 \setminus \{A \I B\} \vdash A \I B$ which contradicts the fact
  that $T_1$ is non-redundant.

  Finally, by $T'_2 \vdash A \I B$ and
  $T'_2 \setminus \{C \I D\} \nvdash A \I B$, we get that
  $T'_2 \vdash A \I C$ by Theorem~\ref{th:AIC} and thus
  $T_1 \vdash A \I C$.
  Moreover, $C \I D \in T'_2$ implies that $T_1 \vdash C \I D$.
  Using the fact that $T_1 \setminus \{A \I B\} \nvdash C \I D$ and
  Theorem~\ref{th:AIC}, we get $T_1 \vdash C \I A$.
  Altogether, $T_1 \vdash A \I C$ and $T_1 \vdash C \I A$
  give $\PrEq[T_1]{A}{C}$.
\end{proof}

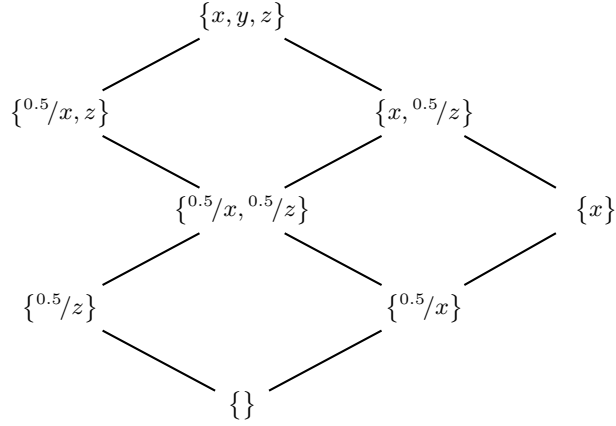
\begin{figure}
  \centering
  \tikzset{graph vertices/.style={%
      nodes={font=\small, inner sep=3pt,minimum width=5em},
      row sep=2em, column sep=1.5em}}%
  \tikzset{graph edges/.style={thick}}%
  \def\SE#1#2{\ensuremath{\bigl\{#1\bigr\}}}
  \begin{tikzpicture}
    \matrix [graph vertices] {
      & \node (a) {\SE{x,y,z}{a}}; \\
      \node (b) {\SE{{}^{0.5\!}/x,z}{b}}; &
      & \node (c) {\SE{x,{}^{0.5\!}/z}{c}}; \\
      & \node (d) {\SE{{}^{0.5\!}/x,{}^{0.5\!}/z}{d}}; &
      & \node (e) {\SE{x}{e}}; \\
      \node (f) {\SE{{}^{0.5\!}/z}{f}}; &
      & \node (g) {\SE{{}^{0.5\!}/x}{g}}; \\
      & \node [minimum width=2em] (h) {\SE{}{h}}; \\
    };
    \path [graph edges]
    (a) edge (b)
    (a) edge (c)
    (b) edge (d)
    (c) edge (d)
    (c) edge (e)
    (d) edge (f)
    (d) edge (g)
    (e) edge (g)
    (f) edge (h)
    (g) edge (h);
  \end{tikzpicture}
  \caption{Lattice of models of theories $T_1$ and $T_2$ from Example~\ref{ex:1}}
  \label{fig:mod}
\end{figure}

\begin{example}\label{ex:2}
  Let us consider the same structure of truth degrees as in Example~\ref{ex:1}
  and let $Y = \{x,y,z\}$. One can check that the following theories
  \begin{align*}
    T_1 =
    \bigl\{&\{{}^{0.5\!}/y\} \I \{x,{}^{0.5\!}/y,z\},
    \{z\} \I \{{}^{0.5\!}/x\},
    \{x,z\} \I \{{}^{0.5\!}/x,y,{}^{0.5\!}/z\}\bigr\},
    \\
    T_2 =
    \bigl\{&\{z\} \I \{{}^{0.5\!}/x\},
    \{{}^{0.5\!}/x,y,{}^{0.5\!}/z\} \I \{x,{}^{0.5\!}/y,{}^{0.5\!}/z\},
    \\[-2pt]
    &\{{}^{0.5\!}/y\} \I \{{}^{0.5\!}/x,{}^{0.5\!}/y,z\},
    \{{}^{0.5\!}/x,{}^{0.5\!}/y\} \I \{y,z\},
    \\[-2pt]
    &\{x,z\} \I \{{}^{0.5\!}/x,y,{}^{0.5\!}/z\}\bigr\}
  \end{align*}
  are non-redundant and equivalent. Indeed, one can check the fact either
  by showing that each formula in $T_1$ is provable by $T_2$ and
  \emph{vice versa} which is straightforward but tedious, or one can show
  that both $T_1$ and $T_2$ have the same models. Figure~\ref{fig:mod}
  shows the set of all models of either of the theories ordered by
  the inclusion of $\mathbf{L}$-sets.

  Now, for each formula in $T_1$ there is a formula in $T_2$ which has
  the same antecedent. Thus, in this direction, the consequence of
  Theorem~\ref{th:PrEq} is immediate. On the other hand, for
  $\{{}^{0.5\!}/x,{}^{0.5\!}/y\} \I \{y,z\} \in T_2$ there
  is no formula in $T_1$ with the antecedent equal to
  $\{{}^{0.5\!}/x,{}^{0.5\!}/y\}$. Nevertheless, we have
  \begin{align*}
    &\PrEq[T_2]{\{{}^{0.5\!}/x,{}^{0.5\!}/y\}}{\{{}^{0.5\!}/y\}}, \\
    &\PrEq[T_2]{\{{}^{0.5\!}/x,{}^{0.5\!}/y\}}{\{x,z\}},
  \end{align*}
  i.e., one can take 
  $\{{}^{0.5\!}/y\} \I \{x,{}^{0.5\!}/y,z\} \in T_1$ or
  $\{x,z\} \I \{{}^{0.5\!}/x,y,{}^{0.5\!}/z\} \in T_1$
  and the same applies to $\{{}^{0.5\!}/x,y,{}^{0.5\!}/z\} \I
  \{x,{}^{0.5\!}/y,{}^{0.5\!}/z\} \in T_2$.

  As in Example~\ref{ex:1}, it can be shown that Theorem~\ref{th:PrEq}
  cannot be extended to general hedges. For instance, consider the following
  theories
  \begin{align*}
    T_3 =
    \bigl\{&\{\} \I \{y,{}^{0.5\!}/z\}\bigr\},
    \\
    T_4 =
    \bigl\{&\{{}^{0.5\!}/y\} \I \{y\},
    \{z\} \I \{{}^{0.5\!}/y,{}^{0.5\!}/z\},
    \{\} \I \{{}^{0.5\!}/z\}\bigr\},
  \end{align*}
  and let $\mathbf{L}$ be the three-element G\"odel chain with ${}^*$
  being the identity. That is, $\mathbf{L}$ is defined as in Example~\ref{ex:1}
  except that $\otimes$ coincides with $\wedge$,
  \begin{align*}
    a \rightarrow b &=
    \left\{
      \begin{array}{@{\,}l@{\quad}l@{}}
        1, & \text{if } a \leq b, \\
        b, & \text{otherwise,}
      \end{array}
    \right.
  \end{align*}
  and $0.5^* = 0.5$. In this setting, $T_3$ and $T_4$ are both
  non-redundant and equivalent. Now,
  for $\{z\} \I \{{}^{0.5\!}/y,{}^{0.5\!}/z\} \in T_4$
  there is no formula in $T_3$ whose antecedent is equivalent to $\{z\}$.
  Indeed, $T_3 \nvdash \{\} \I \{z\}$ on account of soundness
  because $\{y,{}^{0.5\!}/z\}$ is a model of $T_3$.
\end{example}

The relationship between formulas based on equivalence of their
antecedents is crucial for our investigation. We therefore introduce the
following notation. For each $A \in L^Y$ and theory $T$, put
\begin{align}
  \ET{A} &= \{C \I D \in T \,|\, \PrEq{A}{C}\},
  \label{def:ETA}
\end{align}
i.e., $\ET{A}$ is a subset of $T$ containing all formulas whose antecedent
is equivalent to $A$. For particular $A$ and $T$, we may
have $\ET{A} = \emptyset$. The collection of all non-empty subsets
of the form~\eqref{def:ETA} for $A \in L^Y$ is obviously a partition of $T$.
We denote the partition by $\ETs$, i.e.,
\begin{align}
  \ETs &= \{\ET{A} \,|\, \ET{A} \ne \emptyset \text{ and } A \in L^Y\}.
  \label{def:ETs}
\end{align}
The partitions~\eqref{def:ETs} and their classes~\eqref{def:ETA} are
illustrated by the following example.

\begin{example}\label{ex:3}
  Consider theories $T_1$ and $T_2$ from Example~\ref{ex:2} and
  the structure of degrees considered therein.
  In case of $T_1$, the partition $\ETs[T_1]$ given by~\eqref{def:ETs}
  consists of two distinct subsets of $T_1$. Namely,
  \begin{align*}
    \ETs[T_1] &=
    \{
    \{\{{}^{0.5\!}/y\} \I \{x,{}^{0.5\!}/y,z\},
    \{x,z\} \I \{{}^{0.5\!}/x,y,{}^{0.5\!}/z\}\},
    \{\{z\} \I \{{}^{0.5\!}/x\}\}
    \}.
  \end{align*}
  In case of $T_2$, we get:
  \begin{align*}
    \ETs[T_2] =
    \{\{&
    \{{}^{0.5\!}/x,y,{}^{0.5\!}/z\} \I \{x,{}^{0.5\!}/y,{}^{0.5\!}/z\},
    \{{}^{0.5\!}/y\} \I \{{}^{0.5\!}/x,{}^{0.5\!}/y,z\}, \\[-2pt]
    &\{{}^{0.5\!}/x,{}^{0.5\!}/y\} \I \{y,z\},
    \{x,z\} \I \{{}^{0.5\!}/x,y,{}^{0.5\!}/z\}\},
    \{\{z\} \I \{{}^{0.5\!}/x\}\}\}.
  \end{align*}
  Observe that $\ET[T_1]{\{z\}} = \ET[T_2]{\{z\}}$ whereas
  $\ET[T_1]{\{x,z\}} \ne \ET[T_2]{\{x,z\}}$.
\end{example}

\begin{remark}\label{rem:PrEq}
  Obviously, if $\PrEq{A}{B}$, then $\ET{A} = \ET{B}$. Conversely,
  if $\ET{A} = \ET{B} \ne \emptyset$, then there is
  $C \I D \in \ET{A} = \ET{B}$ and thus
  $\PrEq{A}{C}$ and $\PrEq{B}{C}$, i.e., we get $\PrEq{A}{B}$.
  Note that the assumption on $\ET{A}$ and $\ET{B}$ being
  non-empty used in the latter claim cannot be dropped.
  For instance, for $T = \emptyset$ and $A,B \in L^Y$ such
  that $A \nsubseteq B$ and $B \nsubseteq A$,
  we get $\ET{A} = \ET{B} = \emptyset$ and $\emptyset \nvdash A \I B$.
\end{remark}

The following assertions show that despite the fact that equivalent
non-redundant theories can differ in their size, the corresponding
partitions~\eqref{def:ETs} have always the same number of classes.

\begin{theorem}\label{th:ETs}
  Let $T_1$ and $T_2$ be equivalent and non-redundant.
  Then for any $A,B \in L^Y$, the following conditions hold:
  \begin{enumerate}
  \item[\rm (i)]
    If $\ET[T_1]{A} \ne \emptyset$, then $\ET[T_2]{A} \ne \emptyset$.
  \item[\rm (ii)]
    If $\ET[T_1]{A} = \ET[T_1]{B} \ne \emptyset$, then
    $\ET[T_2]{A} = \ET[T_2]{B} \ne \emptyset$.
  \end{enumerate}
\end{theorem}
\begin{proof}
  In order to prove (i), observe that if
  $\ET[T_1]{A} \ne \emptyset$ for $A \in L^Y$, then there is $C \I D \in T_1$
  such that $\PrEq[T_1]{A}{C}$.
  Theorem~\ref{th:PrEq} yields there is $G \I H \in T_2$
  such that $\PrEq[T_1]{C}{G}$. Since $\symbPrEq[T_1]$ is transitive,
  we get $\PrEq[T_1]{A}{G}$. As a consequence, $\PrEq[T_2]{A}{G}$ because
  $T_1$ and $T_2$ are equivalent. Therefore, $G \I H \in \ET[T_2]{A}$
  and so we have $\ET[T_2]{A} \ne \emptyset$.

  Now, (ii) is a consequence of (i) and the argument
  in Remark~\ref{rem:PrEq}. Indeed, if
  $\ET[T_1]{A} = \ET[T_1]{B} \ne \emptyset$, then $\PrEq[T_1]{A}{B}$
  and so $\PrEq[T_2]{A}{B}$ because $T_1$ and $T_2$ are equivalent.
  As a consequence, $\ET[T_2]{A} = \ET[T_2]{B} \ne \emptyset$
  on account of (i).
\end{proof}

\begin{corollary}
  Let $T_1$ and $T_2$ be equivalent and non-redundant theories.
  Then, $|\ETs[T_1]| = |\ETs[T_2]|$.
\end{corollary}
\begin{proof}
  Theorem~\ref{th:ETs} allows us to consider
  a map $h\!: \ETs[T_1] \to \ETs[T_2]$ which is defined
  by $h(\ET[T_1]{A}) = \ET[T_2]{A}$, for each $A \in L^Y$
  such that $\ET[T_1]{A} \ne \emptyset$. Indeed, Theorem~\ref{th:ETs}
  ensures that the map is well defined. In addition, the map is injective  
  because $h(\ET[T_1]{A}) = h(\ET[T_1]{B})$
  means $\ET[T_2]{A} = \ET[T_2]{B} \ne \emptyset$
  and so $\ET[T_1]{A} = \ET[T_1]{B}$ using Theorem~\ref{th:ETs}\,(ii)
  with $T_1$ and $T_2$ interchanged. Hence, $|\ETs[T_1]| = |\ETs[T_2]|$
  follows by the existence of such $h$.
\end{proof}

\begin{remark}
  Example~\ref{ex:3} showed one particular case of two theories $T_1$
  and $T_2$ such that $|T_1| \ne |T_2|$ but $|\ETs[T_1]| = |\ETs[T_2]|$.
  Again, in case of general hedges, the previous observations do not hold.
  As an example, one may take $T_3$ and $T_4$ from Example~\ref{ex:2}
  considering the three-element G\"odel chain with identity as the hedge.
\end{remark}

In order to get further insight into the structure of non-redundant theories,
we introduce a particular notion of provability which is stronger than
the one we have considered so far. The notion is an analog of the direct
determination~\cite{Mai:TRD} established in the framework of the classic
functional dependencies.

\begin{definition}\label{def:dirdet}
  Let $T$ be a theory, $A,B \in L^Y$.
  We say that $A \I B$ is directly provable by $T$, written $\dipr A \I B$,
  whenever
  \begin{align}
    T \setminus \ET{A} \vdash A \I B.
    \label{eqn:dirdet}
  \end{align}
\end{definition}

\begin{remark}
  Obviously, $\dipr[]$ is stronger than $\vdash$.
  If $A \I B$ is trivial then $\dipr A \I B$ for arbitrary $T$
  since for $B \subseteq A$, we have $\emptyset \vdash A \I B$.
  In particular, $\dipr A \I A$. In general, $\dipr[]$ and $\vdash$
  do not coincide. For instance, consider $T = \{A \I B\}$ where $A \I B$
  is non-trivial, i.e., $B \nsubseteq A$. In that case, $\ET{A} = \{A \I B\}$,
  i.e., $T \setminus \ET{A} \nvdash A \I B$, meaning that $\ndipr A \I B$.
\end{remark}

\begin{example}\label{ex:4}
  Take $T_2$ from Example~\ref{ex:2} and let $Y = \{x,y,z\}$.
  The total number of formulas (using $Y$ as the set of all attributes)
  which are provable by $T_2$ is $543$, among those are $327$
  non-trivial ones.
  The number of formulas which are directly provable
  by $T_2$ is considerably lower. Namely, only $231$ formulas are
  directly provable by $T_2$. Moreover, only $15$ of them are
  non-trivial ones. Namely,
  \begin{align*}
    T' = \bigl\{
    &\{{}^{0.5\!}/y,z\} \I \{{}^{0.5\!}/x\},
    \{{}^{0.5\!}/y,z\} \I \{{}^{0.5\!}/x,{}^{0.5\!}/z\},
    \{{}^{0.5\!}/y,z\} \I \{{}^{0.5\!}/x,z\}, \\[-2pt]
    &\{{}^{0.5\!}/y,z\} \I \{{}^{0.5\!}/x,{}^{0.5\!}/y\}, 
    \{{}^{0.5\!}/y,z\} \I \{{}^{0.5\!}/x,{}^{0.5\!}/y,{}^{0.5\!}/z\},
    \\[-2pt]
    &\{{}^{0.5\!}/y,z\} \I \{{}^{0.5\!}/x,{}^{0.5\!}/y,z\},
    \{y,z\} \I \{{}^{0.5\!}/x\},
    \{y,z\} \I \{{}^{0.5\!}/x,{}^{0.5\!}/z\}, \\[-2pt]
    &\{y,z\} \I \{{}^{0.5\!}/x,z\},
    \{y,z\} \I \{{}^{0.5\!}/x,{}^{0.5\!}/y\},
    \{y,z\} \I \{{}^{0.5\!}/x,{}^{0.5\!}/y,{}^{0.5\!}/z\}, \\[-2pt]
    &\{y,z\} \I \{{}^{0.5\!}/x,{}^{0.5\!}/y,z\},
    \{y,z\} \I \{{}^{0.5\!}/x,y\},
    \{y,z\} \I \{{}^{0.5\!}/x,y,{}^{0.5\!}/z\}, \\[-2pt]
    &\{y,z\} \I \{{}^{0.5\!}/x,y,z\}\bigr\}
  \end{align*}
  is the set of all non-trivial formulas which are directly
  provable by $T_2$.
\end{example}

The following assertion shows that theories equivalent in terms of $\vdash$
are also equivalent in terms of $\dipr[]$. That means, when considering
direct provability, one may replace a theory by an equivalent one.

\begin{theorem}\label{th:le55}
  Let $T_1$ and $T_2$ be equivalent.
  If $\dipr[T_1] A \I B$, then $\dipr[T_2] A \I B$.
\end{theorem}
\begin{proof}
  Assume that $\dipr[T_1] A \I B$ and take minimal
  $S \subseteq T_1 \setminus \ET[T_1]{A}$
  such that $S \vdash A \I B$, i.e., $A \I B$ is not provable by
  any proper subset of $S$.
  Now it suffices to show that each formula in $S$ is provable by
  $T_2 \setminus \ET[T_2]{A}$. Indeed, by $S \vdash A \I B$ we then
  conclude that $A \I B$ is provable by $T_2 \setminus \ET[T_2]{A}$.

  Thus, take any $C \I D \in S$. Since $S \vdash A \I B$ and
  $S \setminus \{C \I D\} \nvdash A \I B$, which is a consequence of the
  minimality of $S$, Theorem~\ref{th:AIC} gives $S \vdash A \I C$. That is,
  $T_1 \vdash A \I C$ on account of $S \subseteq T_1$.
  
  By contradiction, assume that $C \I D$ is not provable by
  $T_2 \setminus \ET[T_2]{A}$. Since it is obviously provable by $T_2$
  ($T_2$ is equivalent to $T_1$ and $S \subseteq T_1$), it means that each
  proof of $C \I D$ by $T_2$ contains a formula in $\ET[T_2]{A}$.
  Let $R$ be a minimal subset of $T_2$ such that $R \vdash C \I D$.
  By the minimality of $R$ and utilizing the fact that each proof of $C \I D$
  by $R$ contains a formula in $\ET[T_2]{A}$, it follows that there is
  $G \I H \in \ET[T_2]{A}$ such that
  $R \setminus \{G \I H\} \nvdash C \I D$. By Theorem~\ref{th:AIC},
  we get $R \vdash C \I G$ which further gives $T_2 \vdash C \I G$.
  Moreover, $G \I H \in \ET[T_2]{A}$ gives $T_2 \vdash G \I A$.
  Hence, by $T_2 \vdash C \I G$ and $T_2 \vdash G \I A$,
  we get $T_2 \vdash C \I A$, i.e., $T_1 \vdash C \I A$.
  Taking into account the assumption $T_1 \vdash A \I C$
  from the previous paragraph,
  we conclude that $\PrEq[T_1]{A}{C}$. The latter observation means that
  $C \I D \in \ET[T_1]{A}$ which contradicts the fact that
  $C \I D \in S \subseteq T_1 \setminus \ET[T_1]{A}$.
\end{proof}

\begin{corollary}
  Let $T_1$ and $T_2$ be equivalent.
  Then, for any $A,B \in L^Y$, we have
  $\dipr[T_1] A \I B$ if{}f\/ $\dipr[T_2] A \I B$.
  \qed
\end{corollary}

For any theory $T$, it is easily seen that
by $T \vdash A \I B$ and $T \vdash B \I C$
one can infer $T \vdash A \I C$. This is an immediate consequence of
applying (Tra). An analogous rule of transitivity can also be proved in
case of $\dipr[]$:

\begin{lemma}\label{le:le56}
  If $\dipr A \I B$, and $\dipr B \I C$, then $\dipr A \I C$.
\end{lemma}
\begin{proof}
  Clearly, the claim is trivial if $B \I C$ is a trivial formula,
  i.e., if $C \subseteq B$. Assume that $B \I C$ is non-trivial.
  Observe that if $T \setminus \ET{A} \vdash B \I C$, the claim follows
  directly by (Tra). So, it suffices to show that $B \I C$ is always provable
  by $T \setminus \ET{A}$. By way of contradiction, assume that
  $T \setminus \ET{A} \nvdash B \I C$. Since $T \vdash B \I C$,
  there are $T'$ and $D \I E \in \ET{A}$ such that
  $T \setminus \ET{A} \subset T' \subseteq T$, 
  $T' \vdash B \I C$, and $T' \setminus \{D \I E\} \nvdash B \I C$. 
  Using Theorem~\ref{th:AIC}, it follows that
  $T' \setminus \{D \I E\} \vdash B \I D$
  and so $T \vdash B \I A$ using (Tra) and the monotony of provability
  together with the fact that $T \vdash D \I A$. In addition,
  using $T \vdash A \I B$, we get $\PrEq{A}{B}$, i.e., $\ET{A} = \ET{B}$
  which contradicts our assumption $T \setminus \ET{A} \nvdash B \I C$
  because $\dipr B \I C$ means $T \setminus \ET{B} \vdash B \I C$.
\end{proof}

In the following assertions, we explore antecedents of formulas in $\ET{A}$.
Therefore, for any $A \in L^Y$, we put
\begin{align}
  \et{A} &= \{C \in L^Y \,|\, C \I D \in \ET{A}\}.
\end{align}
As in case of $\ET{A}$, we may have $\et{A} = \emptyset$.

\begin{theorem}\label{th:le57}
  Let $T$ be a theory and let $\et{A} \ne \emptyset$.
  For each $C \in L^Y$ satisfying $\PrEq{A}{C}$ there is
  $D \in \et{A}$ such that $\dipr C \I D$.
\end{theorem}
\begin{proof}
  Take arbitrary $C \in L^Y$ such that $\PrEq{A}{C}$.
  For every $G \in \et{A}$ we get $\PrEq{C}{G}$ and thus $T \vdash C \I G$.
  Take $T' \subseteq T$ and $D \in \et{A}$ with the following property:
  $T' \vdash C \I D$ and if $T'' \vdash C \I G$ for $T'' \subseteq T$
  and $G \in \et{A}$, then $|T'| \leq |T''|$.
  Thus, $T'$ has the minimal size among all theories
  which prove some formula of the form $C \I G$, where $G \in \et{A}$.
  We now show, that $T' \cap \ET{A} = \emptyset$ by which
  we get $T \setminus \ET{A} \vdash C \I D$ yielding
  $\dipr C \I D$.
  
  By way of contradiction, let $G \I H \in T'$ and $G \I H \in \ET{A}$.
  Hence, $G \in \et{A}$ and using the minimality of $T'$,
  we get $T' \setminus \{G \I H\} \nvdash C \I D$.
  Applying Theorem~\ref{th:AIC}, we get $T' \setminus \{G \I H\} \vdash C \I G$,
  i.e., $T' \setminus \{G \I H\}$ contains less formulas than $T'$
  and proves $C \I G$ with $G \in \et{A}$, contradicting the minimality of $T'$.
\end{proof}

\begin{example}\label{ex:5}
  We show non-trivial applications of Theorem~\ref{th:le57}.
  Consider $T_2$ from Example~\ref{ex:2}.
  Take $\{{}^{0.5\!}/x,y,{}^{0.5\!}/z\} \I
  \{x,{}^{0.5\!}/y,{}^{0.5\!}/z\} \in T_2$ and let
  $C = \{{}^{0.5\!}/y,z\}$. Then, for $D = \{{}^{0.5\!}/x,{}^{0.5\!}/y\}$,
  we have $\dipr[T_2] C \I D$. In a more detail, we have
  \begin{align*}
    T_2 \setminus \ET[T_2]{C} &= \{\{z\} \I \{{}^{0.5\!}/x\}\},
  \end{align*}
  cf. Example~\ref{ex:3}. In addition, (Cut) applied to
  $\{{}^{0.5\!}/y,z\} \I \{z\}$ and 
  $\{z\} \I \{{}^{0.5\!}/x\}$
  yields $\{{}^{0.5\!}/y,z\} \I \{{}^{0.5\!}/x\}$ and
  thus $\{{}^{0.5\!}/y,z\} \I \{{}^{0.5\!}/x,{}^{0.5\!}/y\}$
  is provable by $T_2 \setminus \ET[T_2]{C}$, showing 
  $\dipr[T_2] C \I D$.
  Analogously, for $C = \{y,z\}$, we may
  take $D = \{{}^{0.5\!}/x,y,{}^{0.5\!}/z\}$ or
  $D = \{{}^{0.5\!}/x,{}^{0.5\!}/y\}$ and have $\dipr[T_2] C \I D$.
\end{example}

The following assertion is used in the process of finding minimal theories.
It shows that under conditions formulated by equivalence and direct provability,
a pair of formulas in a theory can be equivalently replaced by
a single formula.

\begin{theorem}\label{th:le58}
  Let $T$ be a theory such that
  for $A \I B \in T$ and $C \I D \in T$,
  we have $\PrEq{A}{C}$ and $\dipr A \I C$. Then,
  \begin{align}
    (T \setminus \{A \I B,C \I D\}) \cup \{C \I B \cup D\}
    \label{eqn:le58}
  \end{align}
  is equivalent to $T$.
\end{theorem}
\begin{proof}
  Denote the theory \eqref{eqn:le58} by $T'$.
  Since $A \I B \in T$, we get $T \vdash A \I B$.
  Furthermore, $T \vdash C \I A$ because $\PrEq{A}{C}$.
  Therefore, by (Tra), we get $T \vdash C \I B$.
  Moreover, using the fact that $C \I D \in T$ and applying
  (Add), we get $T \vdash C \I B \cup D$.

  Conversely, it suffices to show that both $A \I B$ and $C \I D$
  are provable by $T'$. Clearly, $T' \vdash C \I D$ results
  by $C \I B \cup D \in T'$ using (Pro). In order to see that
  $A \I B$ is provable by $T'$, observe first that $T' \vdash A \I C$.
  Indeed, $\dipr A \I C$ means that $T \setminus \ET{A} \vdash A \I C$.
  Therefore, taking into account $\PrEq{A}{C}$,
  we get $A \I B \not\in T \setminus \ET{A}$ and 
  $C \I D \not\in T \setminus \ET{A}$, showing
  $T \setminus \{A \I B,C \I D\} \vdash A \I C$ which further
  gives $T' \vdash A \I C$. Now, using $T' \vdash C \I B \cup D$
  and (Tra), we obtain $T' \vdash A \I B \cup D$ and consequently
  $T' \vdash A \I B$ by (Pro).
\end{proof}

By a particular application of Theorem~\ref{th:le58}, we may find
an equivalent theory which consists of formulas with modified antecedents:

\begin{corollary}\label{col:relab}
  Let $T$ be a theory such that $A \I B \in T$, $\PrEq{A}{C}$,
  and $\dipr A \I C$. Then,
  $(T \setminus \{A \I B\}) \cup \{C \I B\}$
  is equivalent to $T$.
\end{corollary}
\begin{proof}
  Take $T' = T \cup \{C \I C\}$.
  By Theorem~\ref{th:le58}, $T'$ is equivalent to 
  \begin{align*}
    (T' \setminus \{A \I B,C \I C\}) \cup \{C \I B \cup C\} =
    (T \setminus \{A \I B\}) \cup \{C \I B \cup C\},
  \end{align*}
  which is equivalent to $(T \setminus \{A \I B\}) \cup \{C \I B\}$
  because $\{C \I B \cup C\} \vdash C \I B$ by (Pro) and 
  $\{C \I B\} \vdash C \I B \cup C$ by the axiom $C \I C$ and (Add).
\end{proof}

\begin{example}\label{ex:6}
  Considering $T_2$ from Example~\ref{ex:2}, there are three pairs
  of formulas $A \I B \in T_2$ and $C \I D \in T_2$ satisfying the conditions
  of Theorem~\ref{th:le58} and which in turn can be used to find a theory
  which is equivalent to $T_2$ and is strictly smaller. Namely, we may
  \begin{itemize}
  \item
    use
    $\{{}^{0.5\!}/x,y,{}^{0.5\!}/z\} \I
    \{x,{}^{0.5\!}/y,{}^{0.5\!}/z\}$
    and $\{{}^{0.5\!}/y\} \I \{{}^{0.5\!}/x,{}^{0.5\!}/y,z\}$, \\
    and replace the formulas by
    $\{{}^{0.5\!}/y\} \I \{x,{}^{0.5\!}/y,z\}$; or
  \item
    use $\{{}^{0.5\!}/x,y,{}^{0.5\!}/z\} \I
    \{x,{}^{0.5\!}/y,{}^{0.5\!}/z\}$
    and $\{{}^{0.5\!}/x,{}^{0.5\!}/y\} \I \{y,z\}$, \\
    and replace the formulas by
    $\{{}^{0.5\!}/x,{}^{0.5\!}/y\} \I \{x,y,z\}$; or
  \item
    use $\{{}^{0.5\!}/x,{}^{0.5\!}/y\} \I \{y,z\}$
    and $\{{}^{0.5\!}/y\} \I \{{}^{0.5\!}/x,{}^{0.5\!}/y,z\}$, \\
    and replace the formulas by
    $\{{}^{0.5\!}/y\} \I \{{}^{0.5\!}/x,y,z\}$.
  \end{itemize}
\end{example}

\begin{lemma}\label{le:th51i}
  Let $T_1$ and $T_2$ be equivalent and non-redundant.
  Then, for each $A \in \et[T_1]{H}$ there is $C \in \et[T_2]{H}$
  such that $\dipr [T_1] A \I C$.
\end{lemma}
\begin{proof}
  Observe that by $A \in \et[T_1]{H}$ and Theorem~\ref{th:PrEq}
  it follows that $\et[T_2]{H} \ne \emptyset$, i.e., there
  is $C' \I D' \in T_2$ such that $\PrEq[T_2]{C'}{H}$ and
  thus $\PrEq[T_2]{C'}{A}$.
  Using Theorem~\ref{th:le57},
  there is $C \in \et[T_2]{C'} = \et[T_2]{H}$ such that $\dipr[T_2] A \I C$.
  Since $T_1$ and $T_2$ are equivalent, $\dipr[T_1] A \I C$
  by Theorem~\ref{th:le55}.
\end{proof}

\begin{example}
  We illustrate the correspondence between antecedents of formulas
  from Lemma~\ref{le:th51i}. Considering theories $T_1$ and $T_2$
  from Example~\ref{ex:2}, for each $A \in \et[T_1]{H}$ there
  is $C \in \et[T_2]{H}$ such that $\dipr[T_1] A \I C$ because all antecedents
  of formulas in $T_1$ are among the antecendents of formulas in $T_2$.
  Conversely, for
  $\{{}^{0.5\!}/x,y,{}^{0.5\!}/z\} \I
  \{x,{}^{0.5\!}/y,{}^{0.5\!}/z\} \in T_2$, we can take
  $\{{}^{0.5\!}/y\} \I \{x,{}^{0.5\!}/y,z\} \in T_1$
  and obviously
  $\dipr[T_2] \{{}^{0.5\!}/x,y,{}^{0.5\!}/z\} \I \{{}^{0.5\!}/y\}$.
  Analogously, for $\{{}^{0.5\!}/x,{}^{0.5\!}/y\} \I \{y,z\} \in T_2$
  there is $\{{}^{0.5\!}/y\} \I \{x,{}^{0.5\!}/y,z\} \in T_1$
  satisfying
  $\dipr[T_2] \{{}^{0.5\!}/x,{}^{0.5\!}/y\} \I \{{}^{0.5\!}/y\}$.
  Also note that Corollary~\ref{col:relab} allows us to modify
  theories while preserving their equivalence.
  For instance, due to our previous observations,
  $\{{}^{0.5\!}/x,y,{}^{0.5\!}/z\} \I
  \{x,{}^{0.5\!}/y,{}^{0.5\!}/z\} \in T_2$ can equivalently be replaced
  by $\{{}^{0.5\!}/y\} \I \{x,{}^{0.5\!}/y,{}^{0.5\!}/z\}$.
\end{example}

We now turn our attention to minimal theories, i.e., theories which
are minimal in terms of the number of formulas:

\begin{definition}
  A theory $T$ is called minimal (in the number of formulas) if for each
  equivalent theory $T'$, we have $|T| \leq |T'|$.
\end{definition}

Obviously, a minimal theory is non-redundant but the converse does not
hold in general. Applying Theorem~\ref{th:le58}, we have the following
corollary.

\begin{corollary}\label{col:le58}
  Let $T$ be minimal. Then there are no distinct $A \I B \in T$ and
  $C \I D \in T$ such that $\PrEq{A}{C}$ and $\dipr A \I C$.
  \qed
\end{corollary}

The following assertions shows properties of direct provability
by minimal theories and their consequences.

\begin{lemma}\label{le:th51}
  Let $T_1$ and $T_2$ be equivalent and minimal.
  Then for 
  $A,A_1,A_2 \in \et[T_1]{H}$ and $C,C_1,C_2 \in \et[T_2]{H}$,
  the following conditions hold:
  \begin{enumerate}
  \item[\rm (i)]
    If $\dipr[T_1] A \I C$,
    % for $A \in \et[T_1]{H}$ and $C \in \et[T_2]{H}$,
    then $\dipr[T_1] C \I A$.
  \item[\rm (ii)]
    If $\dipr[T_1] A \I C_1$ and $\dipr[T_1] A \I C_2$, then $C_1 = C_2$.
  \item[\rm (iii)]
    If $\dipr[T_1] A_1 \I C$ and $\dipr[T_1] A_2 \I C$, then $A_1 = A_2$.
  \end{enumerate}
\end{lemma}
\begin{proof}
  In order to prove (i), we use Lemma~\ref{le:th51i} to conclude that
  for $C \in \et[T_2]{H}$ there is $A' \in \et[T_1]{H}$ such that
  $\dipr[T_2] C \I A'$, i.e., $\dipr[T_1] C \I A'$. 
  Now, using the assumption $\dipr[T_1] A \I C$,
  Lemma~\ref{le:le56} yields $\dipr[T_1] A \I A'$. In addition to that,
  there are $A \I B \in T_1$ and $A' \I B' \in T_1$ with $\PrEq{A}{A'}$.
  Hence, by Corollary~\ref{col:le58}, we get that $A = A'$,
  meaning that $\dipr[T_1] C \I A$.

  In case of (ii), we proceed analogously as in (i).
  By $\dipr[T_1] A \I C_1$,
  we get $\dipr[T_1] C_1 \I A$ by (i) and thus $\dipr[T_1] C_1 \I C_2$ by 
  Lemma~\ref{le:le56}. Then, Corollary~\ref{col:le58} yields $C_1 = C_2$.

  Finally, by $\dipr[T_1] A_2 \I C$ it follows $\dipr[T_1] C \I A_2$ by (i).
  So, analogously as in case of (ii), Lemma~\ref{le:le56} and 
  Corollary~\ref{col:le58} imply
  $\dipr[T_1] A_1 \I A_2$ and thus $A_1 = A_2$, which proves (iii).
\end{proof}

\begin{theorem}\label{th:th51}
  Let $T_1$ and $T_2$ be equivalent and minimal.
  Then, for each $H \in L^Y$, there is an injective map
  $h_H\!: \et[T_1]{H} \to \et[T_2]{H}$.
  Furthermore, $|\et[T_1]{H}| = |\et[T_2]{H}|$.
\end{theorem}
\begin{proof}
  If $\et[T_1]{H}$ is non-empty, then using Lemma~\ref{le:th51i}
  and Lemma~\ref{le:th51}\,(ii) it follows that $h_H$ can be defined
  by $h_H(A) = C$, where $\dipr[T_1] A \I C$. In addition,
  Lemma~\ref{le:th51}\,(iii) gives that $h_H$ is injective.
  Thus, $|\et[T_1]{H}| \leq |\et[T_2]{H}|$.
  The second part follows by application of the claim
  with $T_1$ and $T_2$ interchanged.
\end{proof}

Finally, the next theorem shows that in case of non-redundant theories
which are not minimal, one can always transform the non-redundant theory
into an equivalent and smaller one because the theory contains formulas
satisfying the assumption of Theorem~\ref{th:le58}.

\begin{theorem}\label{th:th52}
  Let $T$ be non-redundant and not minimal.
  Then, there are distinct formulas $A \I B \in T$
  and $C \I D \in T$ such that $\PrEq{A}{C}$ and $\dipr A \I C$.
\end{theorem}
\begin{proof}
  First, observe that if there are $A \I B \in T$ and $C \I D \in T$
  such that $A = C$, then trivially $\PrEq{A}{C}$ and $\dipr A \I C$.
  So, assume that $T$ contains no such formulas. Taking into account
  this assumption, we have $|\ET{H}| = |\et{H}|$ for all $H \in L^Y$.
  Furthermore, let $T'$ be a minimal theory which is equivalent to~$T$.
  Since it is minimal, we also have $|\et[T']{H}| = |\ET[T']{H}|$
  for all $H \in L^Y$.
  
  Since $T$ is not minimal, then there is $H \in L^Y$ such that
  $|\et[T']{H}| = |\ET[T']{H}| < |\ET{H}| = |\et{H}|$ because otherwise
  the non-minimality of $T$ would be violated.
  Using Lemma~\ref{le:th51i}, for any 
  $A_i \in \et{H}$ there is $C_i \in \et[T']{H}$
  such that $\dipr A_i \I C_i$. Considering
  $|\et[T']{H}| < |\et{H}|$ and using the pigeonhole
  principle, there are $A_1,A_2 \in \et{H}$ and $C \in \et[T']{H}$
  such that $A_1 \ne A_2$, $\dipr A_1 \I C$, and $\dipr A_2 \I C$.
  Moreover, for $C$ there is $A' \in \et{H}$
  such that $\dipr[T'] C \I A'$, i.e., $\dipr C \I A'$.
  Since $A_1$ and $A_2$ are distinct,
  either $A_1 \ne A'$ or $A_2 \ne A'$.
  In case of $A_1 \ne A'$, the fact that
  $A_1,A' \in \et{H}$ yields that $\PrEq{A_1}{A}$ and
  there are $A_1 \I B \in T$ and $A' \I D \in T$.
  Furthermore, by $\dipr A_1 \I C$ and $\dipr C \I A'$
  it follows that $\dipr A_1 \I A'$ by~Lemma~\ref{le:le56}.
  Thus, the desired formulas we look for are $A_1 \I B \in T$
  and $A' \I D \in T$. The case of $A_2 \ne A'$ uses the same arguments.
\end{proof}

Based on our observations, we may introduce an algorithm which,
given a theory $T$, finds a theory which is equivalent to $T$ and minimal.
Indeed, one may utilize a standard procedure to find a non-redundant
subset of $T$. That is, one removes all $A \I B \in T$ such that
$T \setminus \{A \I B\} \vdash A \I B$. Then, Corollary~\ref{col:le58}
and Theorem~\ref{th:th52} yield an if and only if condition for $T$
being minimal. Namely, $T$ is minimal if{}f it does not contain 
distinct $A \I B \in T$ and $C \I D \in T$
such that $\PrEq{A}{C}$ and $\dipr A \I C$.
If $T$ contains such formulas, one applies Theorem~\ref{th:le58} to find
a theory which is strictly smaller than $T$. Then, one may repeat
the process until the theory is minimal. The algorithm is illustrated
by the following example.

\begin{example}
  We conclude the examples by applying the previous observations to
  find a minimal theory which is equivalent to $T_2$ from Example~\ref{ex:2}.
  Recall that $T_2$ is non-redundant but it is not minimal
  ($T_1$ from Example~\ref{ex:2} is equivalent to $T_2$ and is strictly
  smaller). Since $T_2$ is not minimal, Theorem~\ref{th:th52} ensures
  there are $A \I B \in T_2$
  and $C \I D \in T_2$ such that $\PrEq[T_2]{A}{C}$ and $\dipr[T_2] A \I C$.
  In particular, we may take 
  $\{{}^{0.5\!}/x,y,{}^{0.5\!}/z\} \I \{x,{}^{0.5\!}/y,{}^{0.5\!}/z\}$
  for $A \I B$ and $\{{}^{0.5\!}/y\} \I \{{}^{0.5\!}/x,{}^{0.5\!}/y,z\}$
  for $C \I D$. Applying Theorem~\ref{th:le58}, $T_2$ can be transformed
  into $T'_2$ of the form
  \begin{align*}
    T'_2 = \{&
    \{{}^{0.5\!}/y\} \I \{x,{}^{0.5\!}/y,z\},
    \{z\} \I \{{}^{0.5\!}/x\}, \\[-2pt]
    &\{{}^{0.5\!}/x,{}^{0.5\!}/y\} \I
    \{y,z\}, \{x,z\} \I \{{}^{0.5\!}/x,y,{}^{0.5\!}/z\}\}
  \end{align*}
  which is strictly smaller than $T_2$. Applying Corollary~\ref{col:le58},
  $T'_2$ is not minimal since for
  $A \I B$ being
  $\{{}^{0.5\!}/x,{}^{0.5\!}/y\} \I \{y,z\}\}$
  and for $C \I D$ being 
  $\{{}^{0.5\!}/y\} \I \{x,{}^{0.5\!}/y,z\}$, we have
  $\PrEq[T'_2]{A}{C}$ and $\dipr[T'_2] A \I C$.
  Therefore, we may apply Theorem~\ref{th:le58} in order
  to transform $T'_2$ into
  \begin{align*}
    T''_2 = \{&
    \{{}^{0.5\!}/y\} \I \{x,y,z\},
    \{z\} \I \{{}^{0.5\!}/x\},
    \{x,z\} \I \{{}^{0.5\!}/x,y,{}^{0.5\!}/z\}\}.
  \end{align*}
  As one can check, $T''_2$ contains no distinct $A \I B$ and $C \I D$
  such that $\PrEq[T''_2]{A}{C}$ and $\dipr[T''_2] A \I C$.
  Hence, by Theorem~\ref{th:th52}, $T''_2$ is minimal.
  Notice that we have derived a minimal equivalent theory $T''_2$
  from $T_2$ without using $T_1$ (from Example~\ref{ex:2}).
  Also, the minimal equivalent theories $T''_2$
  and $T_1$ are distinct.
\end{example}

\begin{remark}
  The the asymptotic time complexity of obtaining a minimal equivalent
  theory is polynomial. Indeed, given a theory $T$, Theorem~\ref{th:th52}
  is applied at most $|T|$ times. In each step, we inspect pairs of formulas
  $A \I B$ and $C \I D$ such that $\PrEq{A}{C}$ and $\dipr A \I C$.
  Both $\PrEq{A}{C}$ and $\dipr A \I C$ can be tested based on computing
  closures, i.e., in time $O(nl)$, where $n$ is the length of $T$
  (total number of attributes appearing in all formulas in $T$)
  and $l$ is the size of the structure of degrees
  (i.e., $l$ is a multiplicative constant depending on $\mathbf{L}$),
  see \textsc{GLinClosure}~\cite{BeVy:GlinClosure}.
  Interestingly, the information on equivalence of $\mathbf{L}$-sets of
  attributes and on direct provability can be computed only once. Indeed,
  since the algorithm transforms theories into equivalent ones,
  by Theorem~\ref{th:le55}, we get that the direct provability of formulas
  is preserved. This makes testing of $\PrEq{A}{C}$ and $\dipr A \I C$
  efficient. Altogether, the algorithm runs in $O(n^2l)$,
  where $n$ is the length of $T$, and $l$ is the size of $\mathbf{L}$.
  This is in contrast with the instance-based approach mentioned in the
  introduction which relies on computing pseudo-intents which is hard
  even in the bivalent case, see~\cite{DiSe:OCEP}.
\end{remark}

\subsubsection*{Conclusion}
We presented an if-and-only-if criterion of minimality of non-redundant set
of graded attribute implications with semantics parameterized by
globalization. The result is constructive and allows to transform
a non-redundant set of graded attribute implications into an equivalent and
minimal one. Issues which we find interesting for future research include
generalization of the approach to accomodate arbitrary linguistic hedges,
construction of efficient algorithms based on the present result,
and connections to other techniques for removing redundancy in both the
classic and graded settings, e.g., the instance-based
approaches like~\cite{GuDu,Zhai}.

\subsubsection*{Acknowledgment}
Supported by grant no. \verb|P202/14-11585S| of the Czech Science Foundation.

%%%%%%%%%%%%%%%%%%%%%%%%%%%%%%%%%%%%%%%%%%%%%%%%%%%%%%%%%%%%%%%%%%%%%%%%%%%%%%%%
%%%%%   BIBLIOGRAPHY
%%%%%%%%%%%%%%%%%%%%%%%%%%%%%%%%%%%%%%%%%%%%%%%%%%%%%%%%%%%%%%%%%%%%%%%%%%%%%%%%

\footnotesize
\bibliographystyle{amsplain}
\bibliography{dirdet}

\providecommand{\bysame}{\leavevmode\hbox to3em{\hrulefill}\thinspace}
\providecommand{\MR}{\relax\ifhmode\unskip\space\fi MR }
\providecommand{\MRhref}[2]{%
  \href{http://www.ams.org/mathscinet-getitem?mr=#1}{#2}
}
\providecommand{\href}[2]{#2}
\begin{thebibliography}{10}

\bibitem{AgImSw:ASR}
Rakesh Agrawal, Tomasz Imieli\'{n}ski, and Arun Swami, \emph{Mining association
  rules between sets of items in large databases}, Proceedings of the 1993 ACM
  SIGMOD International Conference on Management of Data (New York, NY, USA),
  SIGMOD '93, ACM, 1993, pp.~207--216.

\bibitem{Arm:Dsdbr}
William~Ward Armstrong, \emph{Dependency structures of data base
  relationships}, Information Processing 74: Proceedings of IFIP Congress
  (Amsterdam) (J.~L. Rosenfeld and H.~Freeman, eds.), North Holland, 1974,
  pp.~580--583.

\bibitem{Baaz}
Mathias Baaz, \emph{Infinite-valued {G}{\"o}del logics with 0-1 projections and
  relativizations}, G{\"O}DEL '96, Logical Foundations of Mathematics, Computer
  Sciences and Physics (Berlin/Heidelberg), Lecture Notes in Logic, vol.~6,
  Springer-Verlag, 1996, pp.~23--33.

\bibitem{Bel:FRS}
Radim Belohlavek, \emph{Fuzzy {R}elational {S}ystems: {F}oundations and
  {P}rinciples}, Kluwer Academic Publishers, Norwell, MA, USA, 2002.

\bibitem{BeVy:DASFAA}
Radim Belohlavek and Vilem Vychodil, \emph{Data tables with similarity
  relations: Functional dependencies, complete rules and non-redundant bases},
  Database Systems for Advanced Applications (Mong Lee, Kian-Lee Tan, and Vilas
  Wuwongse, eds.), Lecture Notes in Computer Science, vol. 3882, Springer
  Berlin Heidelberg, 2006, pp.~644--658.

\bibitem{BeVy:GlinClosure}
\bysame, \emph{Basic algorithm for attribute implications and functional
  dependencies in graded setting}, International Journal of Foundations of
  Computer Science \textbf{19} (2008), no.~2, 297--317.

\bibitem{BeVy:Fcalh}
\bysame, \emph{Formal concept analysis and linguistic hedges}, International
  Journal of General Systems \textbf{41} (2012), no.~5, 503--532.

\bibitem{BeVy:ADfDwG}
\bysame, \emph{Attribute dependencies for data with grades}, CoRR
  \textbf{abs/1402.2071} (2014), \par\url{http://arxiv.org/abs/1402.2071}.

\bibitem{SL}
Pablo Cordero, \'{A}ngel. Mora, Inmaculada~P\'erez de~Guzm\'an, and Manuel
  Enciso, \emph{Non-deterministic ideal operators: An adequate tool for
  formalization in data bases}, Discrete Applied Mathematics \textbf{156}
  (2008), no.~6, 911--923.

\bibitem{DiSe:OCEP}
Felix Distel and Bar\i\c{s} Sertkaya, \emph{On the complexity of enumerating
  pseudo-intents}, Discrete Appl. Math. \textbf{159} (2011), no.~6, 450--466.

\bibitem{EsGo:MTL}
Francesc Esteva and Llu{\'\i}s Godo, \emph{Monoidal t-norm based logic:
  {T}owards a logic for left-continuous t-norms}, Fuzzy Sets and Systems
  \textbf{124} (2001), no.~3, 271--288.

\bibitem{EsGoNo:Hedges}
Francesc Esteva, Llu\'\i s~Godo, and Carles Noguera, \emph{A logical approach
  to fuzzy truth hedges}, Information Sciences \textbf{232} (2013), 366--385.

\bibitem{GaJiKoOn:RL}
Nikolaos Galatos, Peter Jipsen, Tomacz Kowalski, and Hiroakira Ono,
  \emph{{R}esiduated {L}attices: {A}n {A}lgebraic {G}limpse at {S}ubstructural
  {L}ogics, {V}olume 151}, 1st ed., Elsevier Science, San Diego, USA, 2007.

\bibitem{GaWi:FCA}
Bernhard Ganter and Rudolf Wille, \emph{Formal concept analysis: Mathematical
  foundations}, 1st ed., Springer-Verlag New York, Inc., Secaucus, NJ, USA,
  1997.

\bibitem{Gog:Lic}
Joseph~A. Goguen, \emph{The logic of inexact concepts}, Synthese \textbf{19}
  (1979), 325--373.

\bibitem{Got:Mfl}
Siegfried Gottwald, \emph{Mathematical fuzzy logics}, Bulletin of Symbolic
  Logic \textbf{14} (2008), no.~2, 210--239.

\bibitem{GuDu}
Jean-Louis Guigues and Vincent Duquenne, \emph{Familles minimales
  d'im\-pli\-ca\-tions informatives resultant d'un tableau de donn\'ees
  binaires}, Math. Sci. Humaines \textbf{95} (1986), 5--18.

\bibitem{Haj:MFL}
Petr H\'ajek, \emph{Metamathematics of {F}uzzy {L}ogic}, Kluwer Academic
  Publishers, Dordrecht, The Netherlands, 1998.

\bibitem{Haj:Ovt}
\bysame, \emph{On very true}, Fuzzy Sets and Systems \textbf{124} (2001),
  no.~3, 329--333.

\bibitem{Ho:ML}
Ulrich H\"ohle, \emph{Monoidal logic}, Fuzzy-Systems in Computer Science
  (R.~Kruse, J.~Gebhardt, and R.~Palm, eds.), Artificial Intelligence /
  K\"unstliche Intelligenz, Vieweg+Teubner Verlag, 1994, pp.~233--243.

\bibitem{Hol}
Richard Holzer, \emph{Knowledge acquisition under incomplete knowledge using
  methods from formal concept analysis: Part {I}}, Fundamenta Informaticae
  \textbf{63} (2004), no.~1, 17--39.

\bibitem{Lloyd84}
John~W. Lloyd, \emph{Foundations of {L}ogic {P}rogramming}, Springer-Verlag New
  York, Inc., New York, NY, USA, 1984.

\bibitem{Mai:TRD}
David Maier, \emph{Theory of {R}elational {D}atabases}, Computer Science Pr,
  Rockville, MD, USA, 1983.

\bibitem{Pav:Ofl1}
Jan Pavelka, \emph{On fuzzy logic {I}: {M}any-valued rules of inference},
  Mathematical Logic Quarterly \textbf{25} (1979), no.~3--6, 45--52.

\bibitem{Pav:Ofl2}
\bysame, \emph{On fuzzy logic {II}: {E}nriched residuated lattices and
  semantics of propositional calculi}, Mathematical Logic Quarterly \textbf{25}
  (1979), no.~7--12, 119--134.

\bibitem{Pav:Ofl3}
\bysame, \emph{On fuzzy logic {III}: {S}emantical completeness of some
  many-valued propositional calculi}, Mathematical Logic Quarterly \textbf{25}
  (1979), no.~25--29, 447--464.

\bibitem{TaTi:Gist}
Gaisi Takeuti and Satoko Titani, \emph{Globalization of intuitionistic set
  theory}, Annals of Pure and Applied Logic \textbf{33} (1987), 195--211.

\bibitem{UrVy:Dddosbd}
Lucie Urbanova and Vilem Vychodil, \emph{Derivation digraphs for dependencies
  in ordinal and similarity-based data}, Information Sciences \textbf{268}
  (2014), 381--396.

\bibitem{Za:Afstilh}
Lotfi~A. Zadeh, \emph{A fuzzy-set-theoretic interpretation of linguistic
  hedges}, Journal of Cybernetics \textbf{2} (1972), no.~3, 4--34.

\bibitem{Za:lv1}
\bysame, \emph{The concept of a linguistic variable and its application to
  approximate reasoning--{I}}, Information Sciences \textbf{8} (1975), no.~3,
  199--249.

\bibitem{Za:lv2}
\bysame, \emph{The concept of a linguistic variable and its application to
  approximate reasoning--{II}}, Information Sciences \textbf{8} (1975), no.~4,
  301--357.

\bibitem{Za:lv3}
\bysame, \emph{The concept of a linguistic variable and its application to
  approximate reasoning--{III}}, Information Sciences \textbf{9} (1975), no.~1,
  43--80.

\bibitem{Zak:Mnrar}
Mohammed~J. Zaki, \emph{Mining non-redundant association rules}, Data Mining
  and Knowledge Discovery \textbf{9} (2004), 223--248.

\bibitem{Zhai}
Yanhui Zhai, Deyu Li, and Kaishe Qu, \emph{Decision implication canonical
  basis: a logical perspective}, Journal of Computer and System Sciences
  (2014), \par\url{http://dx.doi.org/10.1016/j.jcss.2014.06.001}.

\end{thebibliography}

\end{document}